\documentclass[12pt,verbose=true,letterpaper,margin=1.5in]{article}
\usepackage{arxiv}

\usepackage[utf8]{inputenc}
\usepackage[linesnumbered,ruled]{algorithm2e}
\usepackage{amsmath,amsthm,amsfonts,mathtools}
\usepackage{tikz}

\usepackage{booktabs} % testing alternative tables
\allowdisplaybreaks % to be able to write long equations on several pages

\usepackage{adjustbox}
\newtheorem{theorem}{Theorem}
\newtheorem{lemma}{Lemma}

\newcommand{\oea}{\mbox{$(1 + 1)$~EA}\xspace}
\newcommand{\mea}{\mbox{$(\mu + 1)$~EA$_D$}\xspace}
\newcommand{\mmea}{\mbox{$(1_\mu + 1_\mu)$~EA$_D$}\xspace}
\newcommand{\mlea}{\mbox{$(\mu + \lambda)$~EA$_D$}\xspace}

\newcommand{\onemax}{\textsc{OneMax}\xspace}

\newcommand{\leadingones}{\textsc{LeadingOnes}\xspace}

\newcommand{\R}{\mathbb{R}}

\newcommand{\Popt}{P_{\text{opt}}}

\DeclareMathOperator{\Geom}{Geom}

\DeclareMathOperator{\Poly}{Poly}

\DeclareMathOperator*{\argmin}{argmin}

\begin{document}

\author{Denis Antipov
\\Optimisation and Logistics\\
School of Computer Science
\\The University of Adelaide
\\Adelaide, Australia
\And
Aneta Neumann
\\Optimisation and Logistics\\
School of Computer Science
\\The University of Adelaide
\\Adelaide, Australia
\And
Frank Neumann
\\Optimisation and Logistics\\
School of Computer Science
\\The University of Adelaide
\\Adelaide, Australia
}

\title{Local Optima in Diversity Optimization: Non-trivial Offspring Population is Essential}

\maketitle

\begin{abstract}
    The main goal of diversity optimization is to find a diverse set of solutions which satisfy some lower bound on their fitness. Evolutionary algorithms (EAs) are often used for such tasks, since they are naturally designed to optimize populations of solutions. This approach to diversity optimization, called EDO, has been previously studied from theoretical perspective, but most studies considered only EAs with a trivial offspring population such as the $(\mu + 1)$~EA. In this paper we give an example instance of a $k$-vertex cover problem, which highlights a critical difference of the diversity optimization from the regular single-objective optimization, namely that there might be a locally optimal population from which we can escape only by replacing at least two individuals at once, which the $(\mu + 1)$ algorithms cannot do.

    We also show that the $(\mu + \lambda)$~EA with $\lambda \ge \mu$ can effectively find a diverse population on $k$-vertex cover, if using a mutation operator inspired by Branson and Sutton (TCS 2023). To avoid the problem of subset selection which arises in the $(\mu + \lambda)$~EA when it optimizes diversity, we also propose the $(1_\mu + 1_\mu)$~EA$_D$, which is an analogue of the $(1 + 1)$ EA for populations, and which is also efficient at optimizing diversity on the $k$-vertex cover problem.
\end{abstract}


\section{Introduction}
\label{sec:intro}

Obtaining a diverse set of good solutions is a complex optimization task, which often arises in real-world problems such as planning~\cite{DBLP:conf/aaai/KatzS20}, satisfiability~\cite{DBLP:conf/gecco/NikfarjamR0023}, architectural planning~\cite{GielL10}, cutting materials~\cite{Haessler1991CuttingSP} and others. The most common reason for the need of a diverse set of solutions is that some objectives or constraints cannot be strictly formalized (e.g., for political, ethical, aesthetic or other reasons), therefore an algorithm user would like to get not a single best solution, but a set of good solutions to choose from. And if this set is not diverse enough, all solutions might occur infeasible in terms of those non-formalizable constraints.

Formalizing diversity is also a non-trivial task, and often it is problem-specific. One of the ways to get a set which can be called diverse is to divide the search space into regions and to optimize the objective (or objectives) in each region simultaneously \cite{CullyM13,DBLP:journals/corr/MouretC15}. This approach, called the quality diversity (or QD for brevity), has been mainly developed in the domains of robotics and games~\cite{DBLP:conf/gecco/PughSSS15,gravina2019procedural,DBLP:conf/gecco/ZardiniZZIF21,chatzilygeroudis2021quality,mkhatshwa2023impact,medina2023evolving}. Recently this
approach has been applied to the traveling thief problem, and various domains such as design and health \cite{nikfarjam2022use,macedo2024evolving,QDATUO_2024}.

Another way to formalize the problem, which we adopt in this paper, is to define a diversity measure over the space of sets of solutions, and turn the problem into optimizing this measure under some constraints on the quality of solutions in the population.

The popularity of this problem has attracted a lot of attention from the algorithmic community, which resulted in developing multiple approaches from deterministic ones~\cite{DBLP:journals/ai/BasteFJMOPR22,DBLP:conf/esa/KellerhalsRZ21} to using various randomized search heuristics~\cite{DBLP:conf/aaai/IngmarBST20,DBLP:conf/ecai/Benke0PL23}. 
One of the most efficient approaches is using evolutionary algorithms (EAs), which is called \emph{evolutionary diversity optimization} or EDO for brevity. It has been used in many applications, including classical combinatorial problems such as the traveling salesperson problem~\cite{DBLP:conf/gecco/NikfarjamBN021,DBLP:journals/telo/DoGNN22}, the traveling thief problem~\cite{DBLP:conf/gecco/NikfarjamN022a}, the computation of minimum spanning trees~\cite{DBLP:conf/gecco/Bossek021}, submodular problems~\cite{DBLP:conf/gecco/NeumannB021} and related communication problems in the area of defense~\cite{DBLP:conf/gecco/NeumannGYSCG023,EDOSDCN}.

The reason behind EDO's spread is that in diversity optimization the aim is to find not a single solution, but a set of solutions, which most optimization algorithms struggle with due to the increased dimensionality of the problem. The EAs, however, have been naturally designed to evolve a population of solutions, hence they do not need much adaptation for the diversity-oriented problems. 

The theoretical foundations of EDO have been established by studying this approach on benchmark functions such as \onemax or \leadingones~\cite{DBLP:conf/gecco/GaoN14}, as well as on permutation problems without any constraint on the fitness of the solutions~\cite{DBLP:journals/telo/DoGNN22}. It has also been studied on submodular functions in~\cite{DBLP:conf/gecco/NeumannB021,DBLP:conf/ijcai/DoGN023}. In all cases upper bounds on the runtime were shown, indicating that the runtime is always finite. It does not appear to be a surprising statement, since most EAs are designed in such a way that they have a finite runtime, e.g., by using standard bit mutation which guarantees a non-zero probability for generating the global optimum in any generation. For this reason in single-objective optimization EAs with a non-trivial parent population such as the $(\mu + 1)$~EA can always find an optimum (probably in ridiculously large $n^n$ time as on the Trap function). 

In this paper we demonstrate that in EDO the standard bit mutation does not guarantee a finite runtime. We give an example of a $k$-vertex cover instance, where none of the algorithms using the $(\mu + 1)$ scheme (that is, the elitist EAs which have a non-trivial parent population and which generate only one offspring per generation) can find the global optimum in finite time. This indicates that in optimization of some diversity measure over a set of populations replacing only one solution is similar to performing a one-bit mutation in single-objective problems, which might get the process stuck in a local optimum on any non-monotone function. 

On the positive side, we show that the \mlea with $\lambda \ge \mu$ can always find an optimal population by generating all individuals from a population with optimal diversity in one iteration. Since selecting a subset with optimal diversity might be a non-trivial task for the \mlea, we also propose the \mmea, which is an analogue of the \oea for population space. We use a jump-and-repair mutation operator from~\cite{DBLP:journals/tcs/BransonS23} and show that the \mlea with $\lambda \ge\mu$ and the \mmea using this operator can find an optimally diverse set of $k$-covers on a graph with $n$ vertices in expected number of at most $2^{k + \mu}(1 - o(1))$ iterations.\footnote{As in the vast majority of theoretical studies, we focus on estimating the number of iterations rather than the wall-clock time.} We note that since the main aim of the diversity optimization is to give a decision-maker (usually, a human) a diverse set of solutions, then $\mu$ is preferred to be not too large, e.g., constant. Therefore, if we also take $k = O(\log(n))$, then the runtime will be polynomial.

The rest of the paper is organized as follows. In Section~\ref{sec:prelims} we formally define the $k$-vertex cover problem, as well as the diversity measures and the algorithms we study in this paper. Section~\ref{sec:local-opt} gives an example of a graph and a population of $k$-vertex covers on this graph which is locally optimal in the space of populations. We then perform a runtime analysis of the \mmea showing its efficiency on this problem in Section~\ref{sec:mmea}. A discussion of our results concludes the paper in Section~\ref{sec:conclusion}.

\section{Preliminaries}
\label{sec:prelims}

\subsection{Diversity Optimization}

The problem of diversity optimization was defined in~\cite{DBLP:conf/gecco/UlrichT11}. We slightly generalize it as follows. Given a fitness function $f: \Omega \mapsto \R$ (where $\Omega$ is the search space), population size $\mu$, quality threshold $B$ and diversity measure $D$, the goal is to find a population $P$, which is a multiset of $\mu$ elements from $\Omega$, with the best value of $D(P)$ under condition that all $x \in P$ meet the quality threshold, that is, $f(x) \ge B$.

Usually the definition of the diversity measure is problem-specific, but it always reflects how different the solutions are in the search space (but not in the fitness space). In this paper we study the \emph{total Hamming distance}, a diversity measure which has been previously studied in the context of theoretical runtime analysis of EDO~\cite{DBLP:conf/gecco/GaoN14,DBLP:conf/gecco/DoerrGN16,DBLP:conf/foga/Antipov0N23}. Given a population $P$ of bit strings, the total Hamming distance is defined as $D(P) = \sum_{x,y \in P} H(x, y)$, where $H$ is the Hamming distance and the sum is over all unique unordered pairs of individuals. 

An important property of this measure for pseudo-Boolean optimization (that is, when the search space is the set $\{0, 1\}^n$ of all bit strings of length $n$) which has been previously used in the analysis of EDO algorithms is that it can be computed through the number of one-bits in each position in the population. It was first shown in~\cite{DBLP:conf/gecco/WinebergO03a}, and later was used in many works which studied this diversity measure. Let $n_i$ be a number of ones in position $i$ over all individuals in population $P$. Then
\begin{align}\label{eq:hamming}
    D(P) = \sum_{i = 1}^{n} n_i (|P| - n_i).
\end{align}
This property implies that the contribution of position $i$ into diversity is maximized, when $n_i = \lceil \frac{|P|}{2} \rceil$ or $n_i = \lfloor \frac{|P|}{2} \rfloor$. The maximum diversity is obtained when all $n_i$ are either $\lceil \frac{|P|}{2} \rceil$ or $\lfloor \frac{|P|}{2} \rfloor$. However it is not always possible due to the constraints on the fitness of the solutions.

\subsection{Vertex Cover}
\label{sec:vc}

For a given undirected graph $G = (V, E)$ we call any set of vertices such that all edges in $E$ are adjacent to at least one vertex in this set a \emph{vertex cover} (or \emph{cover} for short). The \emph{minimum vertex cover} is a problem of finding a cover of minimum size, and the \emph{$k$-vertex cover} is a fixed-target\footnote{For more information about fixed-target analysis and notation see~\cite{DBLP:journals/algorithmica/BuzdalovDDV22}.} variant of this problem, that is, the problem of finding a cover of a size of at most $k$.

Finding a $k$-vertex cover for an arbitrary graph and an arbitrary $k$ is a classic NP-hard problem, and therefore, there is no known algorithm which could solve this problem efficiently (that is, in a polynomial time). For this reason the EAs have been previously applied to it in many different ways. In~\cite{DBLP:conf/cec/OlivetoHY07,DBLP:conf/cec/OlivetoHY08} it has been shown that the classic EAs can be very ineffective on some instances of the vertex cover, when they use a single-objective formulation. Hybrid evolutionary approaches have been studied in~\cite{DBLP:journals/ec/FriedrichHHNW09}, and in~\cite{DBLP:journals/ec/FriedrichHHNW10,DBLP:journals/algorithmica/KratschN13} an effectiveness of the multi-objective approach has been shown. A typical representation of the vertex cover when applying EAs is a bit string of length $n = |V|$, where the $i$-th bit indicates if the $i$-th vertex is included into the set. We use this representation in this paper.  
% In this paper we consider a simplified version of this problem, when all weights are equal, and therefore, we need to find a cover of minimum size.
% \den{Find citations about FPT from Andrew's paper}
% Even this simplified formulation of the problem is NP-hard, hence unless $\P = \NP$, there is no algorithm which can find the optimum in polynomial time. 

It is also well-known that the $k$-vertex cover is a fixed parameter tractable (FPT) problem~\cite{DBLP:journals/algorithmica/KratschN13}, that is, there exists a parameter of the instance $k$ such that the time we need to optimize the instance is $f(k) \cdot \Poly(|V|)$. In our case, this parameter $k$ is the size of the optimal cover.
To address the FPT property of this problem, in~\cite{DBLP:journals/tcs/BransonS23} Branson and Sutton used a modified representation for individuals and proposed a jump-and-repair mutation operator which allowed the \oea to find a $k$-vertex cover in expected number of $O(2^kn^2\log(n))$ iterations, if such cover exists. The main idea behind that operator is that if there exist a vertex cover $y$ of size at most $k$ such that none of the vertices can be removed from it, then we can get it from any other vertex cover $x$ by removing all vertices belonging to $x \setminus y$ and adding their neighbours (see Lemma 4 in~\cite{DBLP:journals/tcs/BransonS23}).

In this work we aim at finding a diverse (in terms of the total Hamming distance) set of vertex covers of size at most $k$ for a given graph $G = (V, E)$, assuming that at least one such cover exists. We also assume that the target population size $\mu$ and the cover size $k$ are relatively small, namely that $k\mu = o(\sqrt{n})$. This implies that by the pigeonhole principle, in any population there will be only $o(\sqrt{n})$ different positions which have at least one one-bit, and therefore there will be many positions $i$, in which all individuals would have a zero-bit (that is, vertex $i$ is not included in any cover in the population). If a population has some individuals which have $\ell < k$ vertices, than adding $k - \ell$ vertices not included into any individual in the population will increase the corresponding terms in eq.~\eqref{eq:hamming} by $|P| = \mu$, hence it never makes sense to have covers of size less than $k$ in the population. 

However, the mutation operator used in~\cite{DBLP:journals/tcs/BransonS23} cannot generate all covers of size $k$, but only \emph{non-excessive} ones, that are, those covers from which we cannot remove any vertex and keep it a cover. For this reason, in this work we modify their mutation operator.
Our modified operator is shown in Algorithm~\ref{alg:jump-and-repair}. Given a cover $x$, this jump-and-repair mutation first removes each vertex from it with probability $1/2$ and then adds all neighbours of the removed vertices, similar to the operator from~\cite{DBLP:journals/tcs/BransonS23}. Since we add all neighbours of the removed vertices, it guarantees that the result of this mutation is a cover, that is, there is no edge for which none of the two adjacent vertices is in the resulting individual. This cover might be of size less than $k$, and in this case we add some randomly chosen vertices to make the size of the vertex cover exactly $k$.

\begin{algorithm}[tp]
  \textbf{Input:} graph $G = (V, E)$ \;
  \textbf{Input:} integer $k > 0$\;
  \textbf{Input:} parent cover $x \subseteq V$ such that $|x| \le k$\;

  $S \gets \emptyset$\tcp*{A set of vertices to remove from $x$}
  \For{$v \in x$}{
    With probability $\frac{1}{2}$ add $v$ to $S$\;
  }
  $y \gets x \setminus S$\;
  \For(\tcp*[f]{Adding neighbours of the removed vertices}){$v \in S$}{
    \For{$u \in V: (u, v) \in E$}{ 
        $y \gets y \cup \{u\}$\;
    }
  }

  \While(\tcp*[f]{Adding more random vertices to get a $k$-cover}){$|y| < k$}{
    \label{line:adding-v}
    Choose $z$ from $V \setminus y$ u.a.r.\;
    $y \gets y \cup \{z\}$\;
  }
  
  \Return $y$\;
 \caption{The jump-and-repair mutation operator for diversity optimization on $k$-vertex cover problem on graph $G = (V, E)$ based on Algorithm~4 in~\cite{DBLP:journals/tcs/BransonS23}.} 
 \label{alg:jump-and-repair}
\end{algorithm}

The following lemma is an extension of Lemma~4 in~\cite{DBLP:journals/tcs/BransonS23}.

\begin{lemma}
\label{lem:operator}
    Let $x$ be a $k$-vertex cover of graph $G$ and let $y$ be a non-excessive cover of size at most $k$. Then the probability that the jump-and-repair mutation (Algorithm~\ref{alg:jump-and-repair}) applied to $x$ generates $y$ before adding additional random vertices (that is, by line~\ref{line:adding-v} in Algorithm~\ref{alg:jump-and-repair}) is exactly $2^{-k}$.
\end{lemma}
\begin{proof}
    Let $S$ be $x \setminus y$, that is, the vertices in $x$ which are not in $y$. All neighbours of vertices in $S$ are in $y$, since otherwise an edge between such a vertex and it neighbour is not covered by $y$. Therefore, removing $S$ from $x$ and adding their neighbours results in a cover which is a subset of $y$. Since $y$ is non-excessive, it is exactly $y$.

    The probability that we remove $S$ and keep $x \setminus S$ is exactly $2^{-k}$, since each vertex is removed with probability $\frac{1}{2}$. \qed
\end{proof}

\subsection{The Considered EAs}

\begin{algorithm}[tp]
  \textbf{Input:} population of $\mu$ individuals $P$\;

  Define $g: x \mapsto \min\{f(x), B\}$\;

  \While{stopping criterion is not met}{
    Create a new individual $y$ \tcp*{We intentionally do not specify how it is created}
    \If{$g(y) \ge \min_{x \in P} g(x)$}{
        $P \gets P \cup \{y\}$\;
        $Q \gets \argmin_{x \in P} g(x)$ \tcp*{$\argmin$ returns a set of individuals (probably, a trivial set of one element)}
        $S \gets \argmin_{x \in Q} D(P \setminus \{x\})$\;
        \If{$|S| \ge 2$ and $y \in S$}{
            $S \gets S \setminus \{y\}$\;
        }
        $x \gets$ random individual from $S$\;
        $P \gets P \setminus \{x\}$\; 
    }
  }
  \Return $P$\;
  
 \caption{The \mea maximizing function $f$ and maximizing diversity measure $D$ with fitness threshold $B$.} 
 \label{alg:mea}
\end{algorithm}

In this paper we consider population-based EAs which are commonly used in the diversity optimization. Most of theoretical studies of EDO considered the \mea, which optimizes the diversity only when it breaks ties between candidates for the next generation~\cite{DBLP:conf/gecco/Bossek021,DBLP:journals/telo/DoGNN22,DBLP:conf/gecco/DoerrGN16}. We describe a generalized version of this EA in Algorithm~\ref{alg:mea}. This algorithm stores a population $P$ of $\mu$ individuals. We do not specify the way these individuals are initialized. In each iteration this algorithm creates a new individual $y$ by applying variation operators (usually, mutation and crossover) to some randomly chosen individuals from the population. If the fitness of $y$ is not worse than the worst fitness in the population or satisfies the quality threshold, $y$ is added into $P$, and then we remove an individual with the worst fitness (counting all individuals which meet the quality threshold as equal). If there are several such individuals, we remove the one with the smallest contribution to the diversity. If there are still more than one such individuals, then we choose one of them uniformly at random (but in this case we do not choose $y$ if it is one of these individuals) and remove it. We note that since we are searching for a population of individuals which meet the minimum fitness threshold $B$, all individuals above this threshold are similarly feasible for us. Hence, instead of optimizing the original objective $f$, the \mea optimizes $g: x \mapsto \min\{f(x), B\}$.

We also consider an elitist EA with a non-trivial offspring population, the \mlea. The main difference of this algorithm from the \mea is that it creates $\lambda$ offspring in each iteration, each by independently choosing parents (or a parent) from the current population and applying variation operators to them. The main complication of the \mlea compared to the \mea is in the selection of the individuals into the next population. After we add all $\lambda$ offspring to the current population $P$, we need to remove $\lambda$ individuals from $P$. We first remove individuals with the minimum fitness (according to the modified fitness $g$) as long as the removal of them does not make size of $P$ less than $\mu$. If after that the size of $P$ is more than $\mu$, we need to break a tie and select $|P|-\mu$ individuals with the worst fitness to remove. We always select a set of $|P| - \mu$ such individuals which has the smallest contribution to the diversity, similar to how we do it in the \mea.

The subset selection problem might be very demanding for the computational resources, especially when we have to break a tie between $2\mu$ individuals, hence in the diversity optimization it makes sense to use the following analogue of the \mlea, which we call the \mmea. This algorithm is initialized with a population of $\mu$ individuals, all of which meet our constraints on the worst fitness value (if we get at least one such individual, we can use a population of its copies). In each iteration it creates a population $P'$ of $\mu$ offspring individuals (in the same way as the \mlea creates $\lambda$ offspring) and replaces $P$ with $P'$, if all individuals of $P'$ are feasible and the diversity of $P'$ is not worse than the diversity of $P$ according to the chosen diversity measure. This algorithm can be considered as a variant of the $(1 + 1)$~EA in the population space, for which individuals are bit strings of size $\mu n$. Its pseudocode is shown in Algorithm~\ref{alg:mmea}.

\begin{algorithm}[tp]
  \textbf{Input:} population of $\mu$ individuals $P$ meeting the minimum fitness threshold\;

  \While{stopping criterion is not met}{
    $P' \gets \emptyset$\;
    \For{$i \in [1..\mu]$}{
        Create a new individual $y$\;
        $P' \gets P' \cup \{y\}$\;
    }
    \If{$\forall x \in P: f(x) \ge B$ and $D(P') \ge D(P)$}{
        $P \gets P'$.
    }
  }
  \Return $P$\;
  
 \caption{The \mmea maximizing diversity measure $D$ under constraint $f(x) \ge B$.} 
 \label{alg:mmea}
\end{algorithm}

\section{Locally Optimal Population}
\label{sec:local-opt}

In this section we give examples of vertex cover instances, for which there exists a population with sub-optimal diversity, such that to improve its diversity we need to change at least two individuals together. This implies that if a $(\mu + 1)$-kind of algorithm gets such a population, it gets stuck in it, since it only changes one individual in each iteration. Note that in this section the diversity is always measured via the total Hamming distance.

We start with a simple example, where the population size is $2$ and the graph has $8$ vertices. We then extend this simple example to an arbitrary even population size $\mu$ and any even problem size $n \ge 10$.

\subsection{The Simple Example}

Consider graph $G = (V, E)$ with $8$ vertices $\{v_1, \dots, v_8\}$ and edges as shown in Figure~\ref{fig:locally-opt-graph}. 

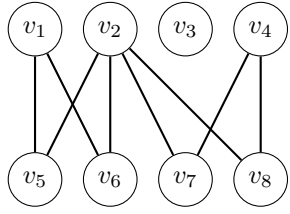
\begin{figure}
    \begin{center}
        \begin{tikzpicture}
            \node (v1) [draw, circle] at (1, 2) {$v_1$};
            \node (v2) [draw, circle] at (2, 2) {$v_2$};
            \node (v3) [draw, circle] at (3, 2) {$v_3$};
            \node (v4) [draw, circle] at (4, 2) {$v_4$};
            \node (v5) [draw, circle] at (1, 0) {$v_5$};
            \node (v6) [draw, circle] at (2, 0) {$v_6$};
            \node (v7) [draw, circle] at (3, 0) {$v_7$};
            \node (v8) [draw, circle] at (4, 0) {$v_8$};

            \draw [thick] (v1) -- (v5);
            \draw [thick] (v1) -- (v6);
            \draw [thick] (v2) -- (v5);
            \draw [thick] (v2) -- (v6);
            \draw [thick] (v2) -- (v7);
            \draw [thick] (v2) -- (v8);
            \draw [thick] (v4) -- (v7);
            \draw [thick] (v4) -- (v8);
        \end{tikzpicture}
    \end{center}
    \caption{Graph $G$, for which we have a population of 4-covers with suboptimal diversity, from which it is impossible to escape by replacing only one individual.}
    \label{fig:locally-opt-graph}
\end{figure}

This is a bipartite graph which has $8$ vertices and $8$ edges. We will show that when we aim at finding the most diverse population of size $\mu = 2$ for a $4$-vertex cover on this graph, we might occur in a local optimum which is impossible to escape for the $(\mu + 1)$ kind of algorithms. The next lemma shows the main properties of $G$.

\begin{lemma}\label{lem:graph-properties}
    The following statements are true for graph $G$ shown in Figure~\ref{fig:locally-opt-graph}
    \begin{enumerate}
        \item There is no vertex cover of size one or two.
        \item The only $3$-vertex cover is $\{v_1, v_2, v_4\}$.
        \item The only $4$-vertex cover which includes $v_3$ is $\{v_1, v_2, v_3, v_4\}$.
        \item The only $4$-vertex cover which does not include $v_2$ is $\{v_5, v_6, v_7, v_8\}$.
    \end{enumerate}
\end{lemma}
\begin{proof}
    The proof of all statements is based on the following observation. If we have a set of vertices $V'$ with known degrees (number of adjacent edges), then the number of edges they cover is $\sum_{v \in V'} deg(v) - K$, where $K$ is the number of edges which are covered from both sides.

    \textbf{The first statement.} 
    Graph $G$ has only one vertex with degree $4$ (namely, $v_2$), and all others have degree 2, except $v_3$ which has degree zero. Hence, one or two vertices can cover at most 4 and 6 edges respectively, which is less than the total number of edges. Hence, no 1-cover or 2-cover exists.

    \textbf{The second statement.}
    A $3$-vertex cover must include $v_2$, otherwise the sum of degrees of the vertices will be at most $6$, which is less than the number of edges. The other two vertices cannot cover any of the edges covered by $v_2$, since otherwise the number of covered edges will be strictly less than the sum of their degrees, that is, 8. Hence, we cannot have $v_5, v_6, v_7$ or $v_8$ in this cover. Among the rest, the only two vertices which have degree at least $2$ are $v_1$ and $v_4$, which we must include into the cover. Since $\{v_1, v_2, v_4\}$ covers all the edges, it is a unique $3$-vertex cover.

    \textbf{The third statement.}
    If we include $v_3$ into the cover, it does not cover any edges. Hence, the rest of the three vertices must cover all $8$ edges. By the previous statement, the only way to do so is to take $\{v_1, v_2, v_4\}$. Hence, $\{v_1, v_2, v_3, v_4\}$ is the only $4$-vertex cover that includes $v_3$.

    \textbf{The fourth statement.}
    If we do not include $v_2$ in the cover, then we need to include all its neighbors to cover the edges adjacent to $v_2$. Its neighbors are $\{v_5, v_6, v_7, v_8\}$, which form a vertex cover. \qed
\end{proof}

Based on these properties we can build a population with locally optimal diversity. We do it in the following lemma. We note that for the population of size two, the diversity is defined by the Hamming distance between the two individuals.

\begin{lemma}
\label{lem:example}
    Consider a problem of finding the most diverse population of size two of 4-vertex covers of graph $G$ shown in Figure~\ref{fig:locally-opt-graph}. The following sets are 4-vertex covers for this graph.
    \begin{itemize}
        \item $V_1 = \{v_1, v_2, v_7, v_8\}$,
        \item $V_2 = \{v_2, v_4, v_5, v_6\}$,
        \item $V_3 = \{v_1, v_2, v_3, v_4\}$,
        \item $V_4 = \{v_5, v_6, v_7, v_8\}$.
    \end{itemize}
    Covers $V_1$, $V_2$ and $V_4$ are non-excessive.
    Then the unique population of size two with the maximum Hamming distance 8 is $(V_3, V_4)$. Population $(V_1, V_2)$ has a Hamming distance 6 and replacing any individual with a different 3- or 4-vertex cover would reduce the Hamming distance, that is, this population is a local optimum.
\end{lemma}

\begin{proof}
    All the sets $V_1$, $V_2$, $V_3$ and $V_4$ are $4$-vertex covers, and the distances between them are the following.
    \begin{itemize}
        \item $H(V_1, V_2) = 6$,
        \item $H(V_3, V_4) = 8$ and
        \item all other distances are $4$ (since each pair coincides in exactly two included vertices and exactly two non-included).
    \end{itemize}

    By Lemma~\ref{lem:graph-properties}, $V_3$ is the unique $4$-cover which has $v_3$, and $V_4$ is the unique $4$-cover which does not include $v_2$. Therefore, for any 3- or 4-cover $V$, covers $V_1$ and $V_2$ coincide with it in including $v_2$ and not including $v_3$, hence $H(V_1, V) \le 6$ and $H(V_2, V) \le 6$. To have $H(V_1, V) = 6$, we need $V$ to be different from $V_1$ in all vertices, except $v_2$ and $v_3$, and the only cover which satisfies it is $V_2$. It is also true the other way around: the only cover in distance 6 from $V_2$ is $V_1$. Hence, if we have population $(V_1, V_2)$, then replacing any of the individuals with another 3- or 4-cover makes the distance between the two individuals smaller, and therefore cannot be accepted by any $(\mu + 1)$ elitist algorithm. \qed
\end{proof}

\subsection{Extending the Example to Arbitrary Population and Problem Sizes}

Based on the example given in the previous subsection, we now show that a population with sub-optimal diversity which cannot be escaped by the $(\mu+1)$~EA exists also for larger $\mu$ and $|V|$. We start with extending our example for larger population sizes, keeping $|V| = 8$.

We consider the same graph $G$ as shown in Figure~\ref{fig:locally-opt-graph}, and give an example of a locally optimal population in the following lemma.

\begin{lemma}
\label{lem:arbitrary-mu}
    Let $\mu \ge 4$ be even and let $\nu = \frac{\mu}{2} - 1$. Let also $V_1$, $V_2$, $V_3$ and $V_4$ be the same vertex covers as in Lemma~\ref{lem:example}. Consider a population which has one individual $V_1$, one individual $V_2$, $\nu$ individuals $V_3$, and $\nu$ individuals $V_4$. Then this population has a sub-optimal diversity (the total Hamming distance) and replacing any individual with any different vertex cover of size at most $4$ reduces the diversity.
\end{lemma}

\begin{proof}
    We exploit the expression of the total Hamming distance given in eq.~\eqref{eq:hamming}. In the given population all $n_i$ (the number of one-bits in position $i$) are $\frac{\mu}{2}$, except for $i = 2$ and $i = 3$. Since the bit strings representing $V_1$ and $V_2$ have a one-bit in position $2$ and a zero-bit in position $3$, we have $n_2 = \frac{\mu}{2} + 1$ and $n_3 = \frac{\mu}{2} - 1$. Hence, if we change the number of one-bits in any position (except $2$ and $3$) by one, it decreases the corresponding term in eq.~\eqref{eq:hamming} by one, since we have
    \begin{align*}
        \left(\frac{\mu}{2}\right)^2 - \left(\frac{\mu}{2}+ 1\right)\left(\frac{\mu}{2} -1\right) = \left(\frac{\mu}{2}\right)^2 - \left(\left(\frac{\mu}{2}\right)^2 - 1\right) = 1.
    \end{align*}
    For the same reason, increasing the number of one-bits in position $2$ decreases this term by $3$ and decreasing the number of one-bits by one increases it by $1$. For position $3$ it is the other way around.
    
    We now consider different cases of replacing individuals in the population with different $4$- or $3$-covers, and show that all of them would only decrease the total Hamming distance.

    \textbf{Case 1: replacing $V_1$.} If we replace the only individual $V_1$ with $V_3$, then we increase the number of one-bits in position $3$ (and therefore, its contribution to the diversity is increased by one), and we change the number of one-bits in positions $4$, $7$ and $8$, which decreases the diversity by $3$. Therefore, the total diversity is decreased. Similarly, if we replace $V_1$ with $V_4$, we make the diversity in position $2$ better, but we unbalance positions $1$, $5$ and $6$, hence, we decrease the diversity.

    If we replace $V_1$ with any other $3$ or $4$-cover, then this cover has the same values in positions $2$ and $3$ by Lemma~\ref{lem:graph-properties}, and at least one other value should be different. This would decrease the term in eq.~\eqref{eq:hamming} which corresponds to this different position. Hence, the diversity is decreased.

    \textbf{Case 2: replacing $V_2$.} This case is similar to Case 1. Replacing $V_2$ with either $V_3$ or $V_4$ decreases the diversity by two, and any other replacement decreases it by at least one.

    \textbf{Case 3: replacing $V_3$.} In this case we consider replacing one of the $\nu$ individuals $V_3$ with a different one. If we replace it with $V_4$, then we decrease $n_2$, which increases the diversity by one. However, it also decreases $n_3$, which decreases the diversity by $3$ and changes all other $n_i$, which decreases the diversity by $6$ more. Hence, the diversity is decreased by $8$. Since by Lemma~\ref{lem:graph-properties} all other covers include $v_2$ and do not include $v_3$, replacing $V_3$ with one of such covers does not change $n_2$ and decreases $n_3$ by one, which reduces the diversity by $3$. The changes in other positions can only reduce diversity even more. Therefore, any replacement of any individual representing $V_3$ would decrease the diversity.

    \textbf{Case 4: replacing $V_4$.} This case is similar to Case 3. If we replace it with $V_3$, we balance position $3$, but we unbalance all other positions, which decreases the diversity. Replacing it with any other cover would unbalance at least one position.

    Bringing all cases together, we conclude that replacing any individual in this population with a different $3$- or $4$-vertex cover decreases the total Hamming distance, and therefore is not accepted by the algorithm. \qed
\end{proof}

To extend this result to larger problem sizes it is enough to add to the graph in Figure~\ref{fig:locally-opt-graph} a complete bipartite graph $K_{n,n}$, which is not connected with the basic graph and consider the $(n + 4)$-cover problem. Then the locally optimal population will be the same as in Lemma~\ref{lem:arbitrary-mu}, but half of the individuals in that population must contain one half of the bipartite graph, and another half of individuals---another half of the bipartite graph. This results in a population in which all positions except 2 and 3 are balanced. Hence, the arguments of Lemma~\ref{lem:arbitrary-mu} will work on this graph as well, if we note that changing the value of any bit corresponding to an added vertex would result in decreasing the corresponding term in~\eqref{eq:hamming} and therefore, reducing the diversity.

\section{Large Offspring Populations Are Effective}
\label{sec:mmea}

In this section we show that the \mmea and also the \mlea with $\lambda \ge \mu$ can effectively find a diverse population of $k$-vertex covers. The main result is the following theorem.

\begin{theorem}
    \label{thm:mmea-runtime}
    Consider the \mmea or the \mlea with $\lambda \ge \mu$, which optimize the total Hamming distance on a $k$-vertex cover instance, for which at least one $k$-cover exist. If we have $k\mu = o(\sqrt{n})$, then in each iteration the probability of these two algorithms to find a population of $k$-vertex covers with optimal diversity is at least $2^{-k\mu}(1 - o(1))$, and therefore, its runtime is dominated by geometric distribution $\Geom(2^{-k\mu}(1 - o(1)))$.
\end{theorem}
\begin{proof}
    Let $\Popt$ be a population of $\mu$ individuals which meet the constraint on fitness and which has the best possible diversity. As it was discussed in Section~\ref{sec:vc}, all individuals in this population have exactly $k$ vertices, since otherwise we could add additional vertices to some of them and increase diversity. Let $\Popt'$ be another population of $\mu$ covers, where the $i$-th individual is a non-excessive cover, which is a subset of the $i$-th individual in $\Popt'$. 

    To get $\Popt$ from $\Popt'$ we need to add vertices to individuals which have less than $k$ vertices. If we add a vertex to position $i$, in which $n_i > 0$ individuals have a one-bit (that is, they include vertex $i$), then the improvement of the corresponding term in eq.~\eqref{eq:hamming} will be 
    \begin{align*}
        (n_i + 1) (\mu - n_i - 1) - n_i (\mu - n_i) = \mu - n_i - 1 - n_i + 1 < \mu.
    \end{align*}
    Hence, we can add this vertex to another position $i$, in which $n_i = 0$ and improve the diversity more, namely, by $\mu$. Therefore, if at least one vertex is added to a position with $n_i > 0$, then it contradicts with optimality of $D(\Popt)$. We also note that any way of adding the missing vertices to $\Popt'$ to positions with $n_i = 0$, would yield the same increase in diversity, therefore, to get an optimal diversity, we do not have to get $\Popt$: any other population obtained in this way from $\Popt'$ has an optimal diversity.

    The probability that we create such a population in one iteration of the \mmea or the \mlea with $\lambda \ge \mu$ is at least the probability that for each $i$ we generate the $i$-th offspring $y_i$ by first creating the $i$-th individual $x_i'$ of $\Popt'$ and then we add the missing vertices (if there are any) to positions with no vertices in other individuals. For the \mmea it gives a population with optimal diversity, and for the \mlea the population with optimal diversity is a subset of the new offspring, hence the diversity of the next-generation population cannot be sub-optimal.

    By Lemma~\ref{lem:operator}, the probability to create an individual from $\Popt'$ is $2^{-k}$, hence the probability that we create exactly $\mu$ such individuals is at least $2^{-k\mu}$. We then add the missing vertices. At each point of time there are at least $n - k\mu$ \emph{good} positions in which there is no vertex in any individual, and there are at most $n$ positions to which we add a vertex, hence the probability that we add it to a good position is at least $(1 - \frac{k\mu}{n})$. Since we need to add at most $k\mu$ vertices, the probability that all of them are added to good positions is at least 
    \begin{align*}
        \left(1 - \frac{k\mu}{n}\right)^{k\mu} \ge 1 - \frac{(k\mu)^2}{n} = 1 - o(1),
    \end{align*}
    where we used Bernoulli inequality and the lemma condition $k\mu = o(\sqrt{n})$. Hence, we conclude that the probability to create a population with optimal diversity in each iteration is at least $2^{-k\mu}(1 - o(1))$. \qed
\end{proof}

If we assume that $\mu$ is some constant, the expected runtime given by Theorem~\ref{thm:mmea-runtime} is at most $O(2^k)$ iterations or fitness evaluations, hence the task of finding a diverse population with the \mmea or the \mlea is easier than finding a $k$-vertex cover with methods, proposed in~\cite{DBLP:journals/tcs/BransonS23}, where an upper bound of $O(2^kn^2\log(n))$ iterations was shown for the \oea.

\section{Conclusion}
\label{sec:conclusion}

In this paper we showed the first example of a local optimum in the diversity optimization problem, from which it is impossible to escape by replacing only one individual in the population if we optimize diversity in an elitist way. This result illustrates that when optimizing in the space of populations, the $(\mu + 1)$ algorithms can be interpreted as local search in that space. To get a positive probability of finding an optimally diverse population in any iteration we have to be able to perform global changes on the population, which demands from us creating at least $\mu$ offspring. This idea brought us to the \mmea, which is an analogue of the \oea, where in role of individuals we have populations represented with bit strings of size $n\mu$.

The first signs of this result might have been seen in the previous empirical study~\cite{DBLP:conf/gecco/NeumannGYSCG023}. There it was shown that the amount of diversity of the obtained solutions in the context of constructing a wireless communication network can be increased in most of the cases by even slightly increasing the offspring population size. This also rises a question on how many individuals should we replace at once to escape such local optima. Creating the number of offspring which is at least the size of the parent population is definitely enough, and, as we have shown in this paper, might be effective when the parent population size is not too large. However, if we want to obtain large diverse populations, then counting on generating the whole optimal population in one iteration is not promising, and our hope would be on improving the diversity via replacing a small number of individuals per iteration.

Before we find an answer to the question on what the offspring population size should be, a lazy approach to such problems would be using variable population size in the \mlea. A good strategy for this might be choosing $\lambda$ according to the power-law distribution, which on the one hand gives us a decent probability to have a large population size, but also preserves a small expected cost of one iteration, as it was shown in~\cite{DBLP:journals/algorithmica/AntipovBD22}.

\textbf{Acknowledgements.} This work has been supported by the Australian Research Council through grants DP190103894 and FT200100536.

% \bibliographystyle{unsrt}
% \bibliography{bibliography}

\end{document}